\def\year{2018}\relax
\documentclass[letterpaper]{article} 
\usepackage{aaai18}  
\usepackage{times}  
\usepackage{helvet}  
\usepackage{courier}  
\usepackage{url}  
\usepackage{graphicx}  

\usepackage[utf8]{inputenc}

\usepackage[]{todonotes}

\usepackage{amssymb}
\usepackage{amsmath}
\usepackage{amsthm}

\newtheorem{theorem}{Theorem}

\theoremstyle{definition}
\newtheorem{definition}{Definition}

\usepackage[nolist,nohyperlinks]{acronym}
\usepackage[inline]{enumitem}

\usepackage{pgf}
\usepackage{tikz}
\usetikzlibrary{arrows,automata}
\usetikzlibrary{shapes,snakes}
\usetikzlibrary{intersections}

\usepackage{mathtools,algorithm}
\usepackage[noend]{algpseudocode}
\usepackage[capitalise]{cleveref}

\makeatletter

\pgfkeys{/pgf/.cd,
  rectangle corner radius/.initial=3pt
}
\newif\ifpgf@rectanglewrc@donecorner@
\def\pgf@rectanglewithroundedcorners@docorner#1#2#3#4{%
  \edef\pgf@marshal{%
    \noexpand\pgfintersectionofpaths
      {%
        \noexpand\pgfpathmoveto{\noexpand\pgfpoint{\the\pgf@xa}{\the\pgf@ya}}%
        \noexpand\pgfpathlineto{\noexpand\pgfpoint{\the\pgf@x}{\the\pgf@y}}%
      }%
      {%
        \noexpand\pgfpathmoveto{\noexpand\pgfpointadd
          {\noexpand\pgfpoint{\the\pgf@xc}{\the\pgf@yc}}%
          {\noexpand\pgfpoint{#1}{#2}}}%
        \noexpand\pgfpatharc{#3}{#4}{\cornerradius}%
      }%
    }%
  \pgf@process{\pgf@marshal\pgfpointintersectionsolution{1}}%
  \pgf@process{\pgftransforminvert\pgfpointtransformed{}}%
  \pgf@rectanglewrc@donecorner@true
}
\pgfdeclareshape{rectangle with rounded corners}
{
  \inheritsavedanchors[from=rectangle] 
  \inheritanchor[from=rectangle]{north}
  \inheritanchor[from=rectangle]{north west}
  \inheritanchor[from=rectangle]{north east}
  \inheritanchor[from=rectangle]{center}
  \inheritanchor[from=rectangle]{west}
  \inheritanchor[from=rectangle]{east}
  \inheritanchor[from=rectangle]{mid}
  \inheritanchor[from=rectangle]{mid west}
  \inheritanchor[from=rectangle]{mid east}
  \inheritanchor[from=rectangle]{base}
  \inheritanchor[from=rectangle]{base west}
  \inheritanchor[from=rectangle]{base east}
  \inheritanchor[from=rectangle]{south}
  \inheritanchor[from=rectangle]{south west}
  \inheritanchor[from=rectangle]{south east}

  \savedmacro\cornerradius{%
    \edef\cornerradius{\pgfkeysvalueof{/pgf/rectangle corner radius}}%
  }

  \backgroundpath{%
    \northeast\advance\pgf@y-\cornerradius\relax
    \pgfpathmoveto{}%
    \pgfpatharc{0}{90}{\cornerradius}%
    \northeast\pgf@ya=\pgf@y\southwest\advance\pgf@x\cornerradius\relax\pgf@y=\pgf@ya
    \pgfpathlineto{}%
    \pgfpatharc{90}{180}{\cornerradius}%
    \southwest\advance\pgf@y\cornerradius\relax
    \pgfpathlineto{}%
    \pgfpatharc{180}{270}{\cornerradius}%
    \northeast\pgf@xa=\pgf@x\advance\pgf@xa-\cornerradius\southwest\pgf@x=\pgf@xa
    \pgfpathlineto{}%
    \pgfpatharc{270}{360}{\cornerradius}%
    \northeast\advance\pgf@y-\cornerradius\relax
    \pgfpathlineto{}%
  }

  \anchor{before north east}{\northeast\advance\pgf@y-\cornerradius}
  \anchor{after north east}{\northeast\advance\pgf@x-\cornerradius}
  \anchor{before north west}{\southwest\pgf@xa=\pgf@x\advance\pgf@xa\cornerradius
    \northeast\pgf@x=\pgf@xa}
  \anchor{after north west}{\northeast\pgf@ya=\pgf@y\advance\pgf@ya-\cornerradius
    \southwest\pgf@y=\pgf@ya}
  \anchor{before south west}{\southwest\advance\pgf@y\cornerradius}
  \anchor{after south west}{\southwest\advance\pgf@x\cornerradius}
  \anchor{before south east}{\northeast\pgf@xa=\pgf@x\advance\pgf@xa-\cornerradius
    \southwest\pgf@x=\pgf@xa}
  \anchor{after south east}{\southwest\pgf@ya=\pgf@y\advance\pgf@ya\cornerradius
    \northeast\pgf@y=\pgf@ya}

  \anchorborder{%
    \pgf@xb=\pgf@x
    \pgf@yb=\pgf@y%
    \southwest%
    \pgf@xa=\pgf@x
    \pgf@ya=\pgf@y%
    \northeast%
    \advance\pgf@x by-\pgf@xa%
    \advance\pgf@y by-\pgf@ya%
    \pgf@xc=.5\pgf@x
    \pgf@yc=.5\pgf@y%
    \advance\pgf@xa by\pgf@xc
    \advance\pgf@ya by\pgf@yc%
    \edef\pgf@marshal{%
      \noexpand\pgfpointborderrectangle
      {\noexpand\pgfqpoint{\the\pgf@xb}{\the\pgf@yb}}
      {\noexpand\pgfqpoint{\the\pgf@xc}{\the\pgf@yc}}%
    }%
    \pgf@process{\pgf@marshal}%
    \advance\pgf@x by\pgf@xa%
    \advance\pgf@y by\pgf@ya%
    \pgfextract@process\borderpoint{}%
    \pgf@rectanglewrc@donecorner@false
    \southwest\pgf@xc=\pgf@x\pgf@yc=\pgf@y
    \advance\pgf@xc\cornerradius\relax\advance\pgf@yc\cornerradius\relax 
    \borderpoint
    \ifdim\pgf@x<\pgf@xc\relax\ifdim\pgf@y<\pgf@yc\relax
      \pgf@rectanglewithroundedcorners@docorner{-\cornerradius}{0pt}{180}{270}%
    \fi\fi
    \ifpgf@rectanglewrc@donecorner@\else
      \southwest\pgf@yc=\pgf@y\relax\northeast\pgf@xc=\pgf@x\relax
      \advance\pgf@xc-\cornerradius\relax\advance\pgf@yc\cornerradius\relax
      \borderpoint
      \ifdim\pgf@x>\pgf@xc\relax\ifdim\pgf@y<\pgf@yc\relax
       \pgf@rectanglewithroundedcorners@docorner{0pt}{-\cornerradius}{270}{360}%
      \fi\fi
    \fi
    \ifpgf@rectanglewrc@donecorner@\else
      \northeast\pgf@xc=\pgf@x\relax\pgf@yc=\pgf@y\relax
      \advance\pgf@xc-\cornerradius\relax\advance\pgf@yc-\cornerradius\relax
      \borderpoint
      \ifdim\pgf@x>\pgf@xc\relax\ifdim\pgf@y>\pgf@yc\relax
       \pgf@rectanglewithroundedcorners@docorner{\cornerradius}{0pt}{0}{90}%
      \fi\fi
    \fi
    \ifpgf@rectanglewrc@donecorner@\else
      \northeast\pgf@yc=\pgf@y\relax\southwest\pgf@xc=\pgf@x\relax
      \advance\pgf@xc\cornerradius\relax\advance\pgf@yc-\cornerradius\relax
      \borderpoint
      \ifdim\pgf@x<\pgf@xc\relax\ifdim\pgf@y>\pgf@yc\relax
       \pgf@rectanglewithroundedcorners@docorner{0pt}{\cornerradius}{90}{180}%
      \fi\fi
    \fi
  }
}

\makeatother

\tikzstyle{line} = [draw, -latex']
\tikzstyle{hidden} = [ellipse, draw, text centered, inner sep=1pt]
\tikzstyle{obs} = [ellipse, draw, fill=gray!60, text centered, inner sep=1pt]
\tikzstyle{rv} = [draw, ellipse, inner sep=2pt]
\tikzstyle{pf} = [draw, rectangle, fill=gray]
\tikzstyle{pc} = [rectangle with rounded corners,draw,rectangle corner radius=15pt, align=center, minimum width=22mm, font=\normalsize, inner sep=3pt]
\tikzstyle{pc2} = [draw, rounded corners=8pt,align=center,minimum height=6mm,font=\normalsize,inner sep=2pt]

\tikzstyle{nhidden} = [draw=none,fill=none,ellipse, text centered, inner sep=1pt]
\tikzstyle{nobs} = [draw=none,fill=none,ellipse, fill=gray!60, text centered, inner sep=1pt]
\tikzstyle{nrv} = [draw=none,fill=none,ellipse, inner sep=2pt]
\tikzstyle{npf} = [draw=none,fill=none,rectangle]

\tikzstyle{ID} = [draw, circle, font=\normalsize]
\tikzstyle{INN} = [draw, circle, inner sep=1pt, fill=black]

\newcommand\pfs[8]{
  \node[pf, #1 of=#2, node distance=#3, xshift=-1mm, yshift=1mm](#6) {};
  \node[pf, #1 of=#2, node distance=#3, label=#4:{#5}](#7) {};
  \node[pf, #1 of=#2, node distance=#3, xshift=1mm, yshift=-1mm](#8) {};
}

\title{Preventing Unnecessary Groundings in the Lifted Dynamic Junction Tree Algorithm}
\author{Marcel Gehrke, Tanya Braun, and Ralf Möller\\ 
Institute of Information Systems, Universität zu Lübeck, Lübeck  \\
\{gehrke, braun, moeller\}@ifis.uni-luebeck.de}
 
\begin{document}

    \acrodef{SHR}[SHR]{Standard Health Record}
\acrodef{ehr}[EHR]{electronic health record}
\acrodef{FHIM}[FHIM]{Federal Health Information Model}
\acrodef{OHDSI}[OHDSI]{OMOP Common Data Model, the Observational Health Data Sciences and Informatics}
\acrodef{FHIR}[FHIR]{HL7’s FAST Healthcare Interoperability Resources}

\acrodef{jtree}[jtree]{junction tree}
\acrodef{plms}[PLMs]{probabilistic logical models}

\acrodef{pdb}[PDB]{probabilistic database}

\acrodef{pf}[parfactor]{parametric factor}
\acrodef{lv}[logvar]{logical variable}
\acrodef{rv}[randvar]{random variable}
\acrodef{crv}[CRV]{counting randvar}
\acrodef{prv}[PRV]{parameterised randvar}

\acrodef{fodt}[FO dtree]{first-order decomposition tree}

\acrodef{fojt}[FO jtree]{first-order junction tree}
\acrodef{ljt}[LJT]{lifted junction tree algorithm}
\acrodef{ldjt}[LDJT]{lifted dynamic junction tree algorithm}
\acrodef{lve}[LVE]{lifted variable elimination}

\acrodef{2tpm}[2TPM]{two-slice temporal parameterised model}
\acrodef{2tbn}[2TBN]{two-slice temporal bayesian network}

\acrodef{pm}[PM]{parameterised probabilistic model}
\acrodef{pdecm}[PDecM]{parameterised probabilistic decision model}

\acrodef{pdm}[PDM]{parameterised probabilistic dynamic model}
\acrodef{pddecm}[PDDecM]{parameterised probabilistic dynamic decision model}

\acrodef{dbn}[DBN]{dynamic Bayesian network}
\acrodef{bn}[BN]{Bayesian network}

\acrodef{dfg}[DFG]{dynamic factor graph}

\acrodef{dmln}[DMLN]{dynamic Markov logic network}

\acrodef{rdbn}[RDBN]{relational dynamic Bayesian network}

\acrodef{meu}[MEU]{maximum expected utility}
\acrodef{mldn}[MLDN]{Markov logic decision network}

\maketitle
    \begin{abstract}
    The \ac{ldjt} efficiently answers filtering and prediction queries for probabilistic relational temporal models by building and then reusing a first-order cluster representation of a knowledge base for multiple queries and time steps.
    Unfortunately, a non-ideal elimination order can lead to groundings even though a lifted run is possible for a model.
    We extend \ac{ldjt}
    \begin{enumerate*}
        \item to identify unnecessary groundings while proceeding in time and
        \item to prevent groundings by delaying eliminations through changes in a temporal first-order cluster representation. 
    \end{enumerate*}   
    The extended version of LDJT answers multiple temporal queries orders of magnitude faster than the original version.
\end{abstract}
\acresetall	

    \section{Introduction}\label{sec:intro}

Areas like healthcare, logistics or even scientific publishing deal with probabilistic data with relational and temporal aspects and need efficient exact inference algorithms.
These areas involve many objects in relation to each other with changes over time and uncertainties about object existence, attribute value assignments, or relations between objects.  
More specifically, publishing involves publications (the relational part) for many authors (the objects), streams of papers over time (the temporal part), and uncertainties for example due to missing or incomplete information.
By performing model counting, \acp{pdb} can answer queries for relational temporal models with uncertainties~\cite{dignos2012temporal,dylla2013temporal}.
However, each query embeds a process behaviour, resulting in huge queries with possibly redundant information.
In contrast to \acp{pdb}, we build more expressive and compact models including behaviour (offline) enabling efficient answering of more compact queries (online).
For query answering, our approach performs deductive reasoning by computing marginal distributions at discrete time steps.
In this paper, we study the problem of exact inference and investigate how to prevent unnecessary groundings in large temporal probabilistic models that exhibit symmetries.

We propose \acp{pdm} to represent probabilistic relational temporal behaviour and introduce the \ac{ldjt} to exactly answer multiple filtering and prediction queries for multiple time steps efficiently \cite{gehrke2018ldjt}.
\ac{ldjt} combines the advantages of the interface algorithm~\cite{Murphy:2002:DBN} and the \ac{ljt}~\cite{BrMoe16a}.
Specifically, this paper extends \ac{ldjt} and contributes
\begin{enumerate*}
    \item means to identify whether groundings occur and
    \item an approach to prevent unnecessary groundings by extending inter \ac{fojt} separators.
\end{enumerate*} 

\ac{ldjt} reuses an \ac{fojt} structure to answer multiple queries and reuses the structure to answer queries for all time steps $t > 0$.
Additionally, \ac{ldjt} ensures a minimal exact inter \ac{fojt} information propagation over a separator. 
Unfortunately, due to a non-ideal elimination order unnecessary groundings can occur.
In the static case, \ac{ljt} prevents groundings by fusing parclusters, the nodes of an \ac{fojt}.
For the temporal case, fusing parclusters is not applicable, as \ac{ldjt} would need to fuse parclusters of different \acp{fojt}. 
We propose to prevent groundings by extending inter \ac{fojt} separators and thereby changing the elimination order by delaying eliminations to the next time step.

The remainder of this paper has the following structure: 
We begin by recapitulating \acp{pdm} as a representation for relational temporal probabilistic models and present \ac{ldjt}, an efficient reasoning algorithm for \acp{pdm}.
Afterwards, we present \ac{ljt}'s techniques to prevent unnecessary groundings and extend \ac{ldjt} to prevent unnecessary groundings.
Lastly, we evaluate the extended version of \ac{ldjt} against \ac{ldjt}'s orignal version and \ac{ljt}.
We conclude by looking at possible extensions.

    \section{Related Work}\label{sec:rel}
We take a look at inference for propositional temporal models, relational static models, and give an overview about relational temporal model research.

For exact inference on propositional temporal models, a naive approach is to unroll the temporal model for a given number of time steps and use any exact inference algorithm for static, i.e., non-temporal, models.
In the worst case, once the number of time steps changes, one has to unroll the model and infer again.
Murphy~\shortcite{Murphy:2002:DBN} proposes the interface algorithm consisting of a forward and backward pass that uses a temporal d-separation with a minimal set of nodes to apply static inference algorithms to the dynamic model.

First-order probabilistic inference leverages the relational aspect of a static model.
For models with known domain size, first-order probabilistic inference exploits symmetries in a model by combining instances to reason with representatives, known as lifting \cite{poole2003first}. 
Poole~\shortcite{poole2003first} introduces parametric factor graphs as relational models and proposes \ac{lve} as an exact inference algorithm on relational models.
Further, de Salvo Braz~\shortcite{Braz07}, Milch et al.~\shortcite{milch2008lifted}, and Taghipour et al.~\shortcite{TagFiDaBl13} extend \ac{lve} to its current form.
Lauritzen and Spiegelhalter~\shortcite{lauritzen1988local}  introduce the junction tree algorithm.
To benefit from the ideas of the junction tree algorithm and \ac{lve}, Braun and Möller~\shortcite{BrMoe16a} present \ac{ljt}, which efficiently performs exact first-order probabilistic inference on relational models given a set of queries.

To handle inference for relational temporal models most approaches are approximative.
Additional to being approximative, these approaches involve unnecessary groundings or are only designed to handle single queries efficiently.
Ahmadi et al.~\shortcite{ahmadi2013exploiting} propose lifted (loopy) belief propagation.
From a factor graph, they build a compressed factor graph and apply lifted belief propagation with the idea of the factored frontier algorithm \cite{murphy2001factored}, which is an approximate counterpart to the interface algorithm. 
Thon et al.~\shortcite{thon2011stochastic} introduce CPT-L, a probabilistic model for sequences of relational state descriptions with a partially lifted inference algorithm.
Geier and Biundo~\shortcite{geier2011approximate} present an online interface algorithm for \acp{dmln}, similar to the work of Papai et al.~\cite{papai2012slice}.
Both approaches slice \acp{dmln} to run well-studied static MLN \cite{richardson2006markov} inference algorithms on each slice individually.  
Two ways of performing online inference using particle filtering are described in \cite{manfredotti2009modeling,nitti2013particle}.

Vlasselaer et al.~\shortcite{vlasselaer2014efficient,vlasselaer2016tp} introduce an exact approach, which involves computing probabilities of each possible interface assignment on a ground level.
    \section{Parameterised Probabilistic Models}\label{sec:back}
Based on \cite{BraMo17}, we present \acp{pm} for relational static models.
Afterwards, we extend \acp{pm} to the temporal case, resulting in \acp{pdm} for relational temporal models, which, in turn, are based on \cite{gehrke2018ldjt}. 

\subsection{Parameterised Probabilistic Models}\label{pm}

\acp{pm} combine first-order logic with probabilistic models, representing first-order constructs using \acp{lv} as parameters.
\begin{definition}
    Let $\mathbf{L}$ be a set of \ac{lv} names, $\Phi$ a set of factor names, and $\mathbf{R}$ a set of \ac{rv} names. 
    A \ac{prv} $A = P(X^1,...,X^n)$ represents a set of \acp{rv} behaving identically by combining a \ac{rv} $P \in \mathbf{R}$ with $X^1,...,X^n \in \mathbf{L}$.
    If $n = 0$, the \ac{prv} is parameterless.
    The domain of a \ac{lv} $L$ is denoted by $\mathcal{D}(L)$.
    The term $range(A)$ provides possible values of a \ac{prv} $A$.
    Constraint $(\mathbf{X},C_\mathbf{X})$ allows to restrict \acp{lv} to certain domain values and is a tuple with a sequence of \acp{lv} $\mathbf{X} = (X^1,...,X^n)$ and a set $C_\mathbf{X} \subseteq \times_{i=1}^n \mathcal{D}(X^i)$.
    $\top$ denotes that no restrictions apply and may be omitted.
    The term $lv(Y)$ refers to the \acp{lv} in some element $Y$. 
    The term $gr(Y)$ denotes the set of instances of $Y$ with all \acp{lv} in $Y$ grounded w.r.t.\ constraints.
\end{definition}

Let us set up a \ac{pm} for publications on some topic. 
We model that the topic may be hot, conferences are attractive, people do research, and publish in publications.
From $\mathbf{R} = \{Hot, DoR\}$ and $\mathbf{L} = \{A, P, X\}$ with $\mathcal{D}(A) = \{a_1, a_2\}$, $\mathcal{D}(P) = \{p_1, p_2\}$, and $\mathcal{D}(X) = \{x_1, x_2, x_3\}$, we build the boolean PRVs $Hot$ and $DoR(X)$.
With $C = (X, \{x_1, x_2\})$, $gr(DoR(X)|C) = \{DoR(x_1), DoR(x_2)\}$.

\begin{definition}
We denote a \ac{pf} $g$ with
$\forall \mathbf{X} : \phi(\mathcal{A})\;| C$.
$\mathbf{X} \subseteq \mathbf{L}$ being a set of \acp{lv} over which the factor generalises and $\mathcal{A} = (A^1,...,A^n)$ a sequence of \acp{prv}.
We omit $(\forall \mathbf{X} :)$ if $\mathbf{X} = lv(\mathcal{A})$.
A function $\phi : \times_{i=1}^n range(A^i) \mapsto \mathbb{R}^+$ with name $\phi \in \Phi$ is defined identically for all grounded instances of $\mathcal{A}$.
A list of all input-output values is the complete specification for $\phi$.
$C$ is a constraint on $\mathbf{X}$.
A \ac{pm} $G := \{g^i\}_{i=0}^{n-1}$ is a set of \acp{pf} and semantically represents the full joint probability distribution $P(G) = \frac{1}{Z} \prod_{f \in gr(G)} \phi(\mathcal{A}_f)$ where $Z$ is a normalisation constant.
\end{definition}

\begin{figure}[t]
\centering
\scalebox{0.9}{
\begin{tikzpicture}[rv/.style={draw, ellipse, inner sep=1pt},pf/.style={draw, rectangle, fill=gray},label distance=0.2mm]
	\node[rv]					 								(S)	{$Hot$};
    \pfs{left}{S}{20mm}{230}{$g^0$}{USa}{US}{USb}    
	\node[rv, left of=US, node distance=20mm, inner sep=0.5pt]			(U)	{$Pub(X,P)$};    
	\node[obs, below of=S, node distance=10mm]						(T1)	{$AttC(A)$};    
    \pfs{left}{T1}{20mm}{270}{$g^1$}{ASa}{AS}{ASb}    
	\node[rv, left of=AS, node distance=20mm, inner sep=0.5pt]			(A)	{$DoR(X)$};    
    
	\draw (U) -- (US);
	\draw (US) -- (S);
	\draw (US) -- (T1);

	\draw (A) -- (AS);
	\draw (AS) -- (S);
	\draw (AS) -- (T1);

\end{tikzpicture}
}
\caption{Parfactor graph for $G^{ex}$ }
\label{fig:swe}
\end{figure}
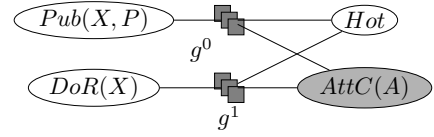%

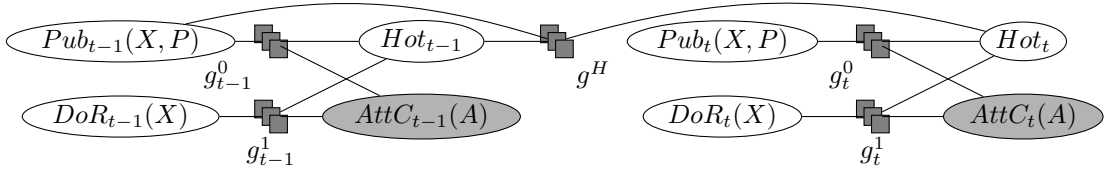
\begin{figure*}[t]
\center
\scalebox{1}{
\begin{tikzpicture}[rv/.style={draw, ellipse, inner sep=1pt},pf/.style={draw, rectangle, fill=gray},label distance=0.2mm]
	\node[rv]					 								(S)	{$Hot_{t-1}$};
    \pfs{left}{S}{20mm}{230}{$g^0_{t-1}$}{USa}{US}{USb}    
	\node[rv, left of=US, node distance=20mm, inner sep=0.5pt]			(U)	{$Pub_{t-1}(X,P)$};    
	\node[obs, below of=S, node distance=10mm]						(T1)	{$AttC_{t-1}(A)$};    
    \pfs{left}{T1}{20mm}{270}{$g^1_{t-1}$}{ASa}{AS}{ASb}    
	\node[rv, left of=AS, node distance=20mm, inner sep=0.5pt]			(A)	{$DoR_{t-1}(X)$};    
    
	\node[rv, right of = S, node distance=80mm]					 		(S1)	{$Hot_{t}$};
    \pfs{left}{S1}{20mm}{230}{$g^0_{t}$}{USa}{US1}{USb}    
	\node[rv, left of=US1, node distance=20mm, inner sep=0.5pt]			(U1)	{$Pub_{t}(X,P)$};    
	\node[obs, below of=S1, node distance=10mm]						(T11)	{$AttC_{t}(A)$};    
    \pfs{left}{T11}{20mm}{270}{$g^1_{t}$}{ASa}{AS1}{ASb}    
	\node[rv, left of=AS1, node distance=20mm, inner sep=0.5pt]			(A1)	{$DoR_{t}(X)$};

    
    \pfs{right}{S}{18mm}{315}{$g^H$}{UAa}{IU}{UAb}

    \path [-, bend right=15] (IU) edge node {} (U);
    \path [-, bend left=15] (IU) edge node {} (S1);
	\draw (IU) -- (S);
    
	\draw (U) -- (US);
	\draw (US) -- (S);
	\draw (US) -- (T1);

	\draw (A) -- (AS);
	\draw (AS) -- (S);
	\draw (AS) -- (T1);

	\draw (U1) -- (US1);
	\draw (US1) -- (S1);
	\draw (US1) -- (T11);

	\draw (A1) -- (AS1);
	\draw (AS1) -- (S1);
	\draw (AS1) -- (T11);


\end{tikzpicture}
}
\caption{$G_\rightarrow^{ex}$ the two-slice temporal parfactor graph for model $G^{ex}$}
\label{fig:TSPG}	
\end{figure*}

Adding boolean PRVs $Pub(X,P)$ and $AttC(A)$,
$G_{ex} = \{g^i\}^1_{i=0}$,
	$g^0 = \phi^0(Pub(X,P), AttC(A), Hot)$, 
	$g^1 = \phi^1(DoR(X), AttC(A), Hot)$
forms a model.
All \acp{pf} have eight input-output pairs (omitted). 
Constraints are $\top$, i.e., the $\phi$'s hold for all domain values.
E.g., $gr(g^1)$ contains four factors with identical $\phi$.
\Cref{fig:swe} depicts $G^{ex}$ as a graph with four variable nodes for the PRVs and two factor nodes for $g^0$ and $g^1$ with edges to the PRVs involved.
Additionally, we can observe the attractiveness of conferences.
The remaining \acp{prv} are latent.

The semantics of a model is given by grounding and building a full joint distribution.
In general, queries ask for a probability distribution of a \ac{rv} using a model's full joint distribution and fixed events as evidence. 

\begin{definition}
    Given a \ac{pm} $G$, a ground \ac{prv} $Q$ and grounded \ac{prv}s with fixed range values $\mathbf{E}$, the expression $P(Q|\mathbf{E})$ denotes a query w.r.t.\ $P(G)$.
\end{definition}


\subsection{Parameterised Probabilistic Dynamic Models}\label{sec:pdm}

To define \acp{pdm}, we use \acp{pm} and the idea of how \aclp{bn} give rise to \aclp{dbn}. 
We define \acp{pdm} based on the first-order Markov assumption, i.e., a time slice $t$ only depends on the previous time slice $t-1$. 
Further, the underlining process is stationary, i.e., the model behaviour does not change over time. 

\begin{definition}
    A \ac{pdm} is a pair of \acp{pm} $(G_0,G_\rightarrow)$ where
        $G_0$ is a PM representing the first time step and  
        $G_\rightarrow$ is a \acl{2tpm} representing $\mathbf{A}_{t-1}$ and $\mathbf{A}_t$ where
    $\mathbf{A}_\pi$ is  a set of \acp{prv} from time slice $\pi$.
\end{definition}

\Cref{fig:TSPG} shows how the model $G^{ex}$ behaves over time.
$G_\rightarrow^{ex}$ consists of $G^{ex}$ for time step $t-1$ and for time step $t$ with inter-slice \ac{pf} for the behaviour over time. 
In this example, the \ac{pf} $g^{H}$ is the inter-slice \acp{pf}. 

\begin{definition}
    Given a \ac{pdm} $G$, a ground \ac{prv} $Q_t$ and grounded \ac{prv}s with fixed range values $\mathbf{E}_{0:t}$ the expression $P(Q_t|\mathbf{E}_{0:t})$ denotes a query w.r.t.\ $P(G)$.
\end{definition}

The problem of answering a marginal distribution query $P(A^i_\pi|\mathbf{E}_{0:t})$ w.r.t.\ the model is called prediction for $\pi > t$ and filtering for $\pi = t$. 

    \section{Lifted Dynamic Junction Tree Algorithm}\label{sec:fodjt}
To provide means to answer queries for \acp{pm}, we introduce \ac{ljt}, mainly based on \cite{braun2017preventing}.
Afterwards, we present \ac{ldjt} \cite{gehrke2018ldjt} consisting of \ac{fojt} constructions for a \ac{pdm} and a filtering and prediction algorithm. 

\subsection{Lifted Junction Tree Algorithm}\label{ljt}

\ac{ljt} provides efficient means to answer queries $P(\mathbf{Q}|\mathbf{E})$, with a set of query terms, given a \ac{pm} $G$ and evidence $\mathbf{E}$, by performing the following steps:
\begin{enumerate*}
    \item Construct an \ac{fojt} $J$ for $G$.
    \item Enter $\mathbf{E}$ in $J$.
    \item Pass messages.
    \item Compute answer for each query $Q^i \in \mathbf{Q}$.
\end{enumerate*}
We first define an \ac{fojt} and then go through each step. 
To define an \ac{fojt}, we need to define parameterised clusters (parclusters), the nodes of an \ac{fojt}.

\begin{definition}
    A parcluster $\mathbf{C}$ is defined by $\forall \mathbf{L} : \mathbf{A} | C$.
    $\mathbf{L}$ is a set of \ac{lv}s, $\mathbf{A}$ is a set of \ac{prv}s with $lv(\mathbf{A}) \subseteq \mathbf{L}$, and $C$ a constraint on $\mathbf{L}$.
    We omit $(\forall \mathbf{L} :)$ if $\mathbf{L} = lv(\mathbf{A})$.
    A parcluster $\mathbf{C}^i$ can have parfactors $\phi(\mathcal{A}^\phi) | C^\phi $ assigned given that
    \begin{enumerate*}
        \item $\mathcal{A}^\phi \subseteq \mathbf{A}$,
        \item $lv(\mathcal{A}^\phi) \subseteq \mathbf{L}$, and
        \item $C^\phi \subseteq C$
    \end{enumerate*}
    holds.
    We call the set of assigned \ac{pf}s a local model $G^i$.
\\
An \ac{fojt} for a model $G$ is $J=(\mathbf{V},\mathbf{E})$ where $J$ is a cycle-free graph,
the nodes $\mathbf{V}$ denote a set of parcluster, and the set $\mathbf{E}$ edges between parclusters. 
An \ac{fojt} must satisfy the following properties:
\begin{enumerate*}
    \item A parcluster $\mathbf{C}^i$ is a set of \acp{prv} from $G$.
    \item For each \ac{pf} $\phi(\mathcal{A}) | C$  in G, $\mathcal{A}$ must appear in some parcluster $\mathbf{C}^i$.
    \item If a \ac{prv} from $G$ appears in two parclusters $\mathbf{C}^i$ and $\mathbf{C}^j$, it must also appear in every parcluster $\mathbf{C}^k$ on the path connecting nodes i and j in $J$.
\end{enumerate*}
The separator $\mathbf{S}^{ij}$ of edge $i-j$ is given by $\mathbf{C}^i \cap \mathbf{C}^j$ containing shared \acp{prv}. 
\end{definition}

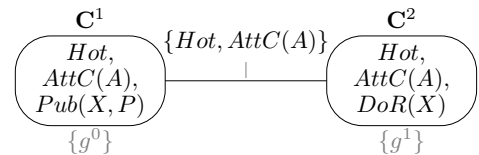
\begin{figure}[b]
    \centering
    \scalebox{0.9}{
\begin{tikzpicture}[node distance=46mm]
	\node[pc, label={[gray, inner sep=1pt]270:{$\{g^0\}$}},label={[font=]90:{$\mathbf{C}^1$}}]				(c1) {$Hot,$\\$ AttC(A),$\\ $Pub(X,P)$};
	\node[pc, right of=c1, label={[gray, inner sep=1pt]270:{$\{g^1\}$}},label={90:{$\mathbf{C}^2$}}]	(c2) {$Hot,$\\$ AttC(A),$ \\ $DoR(X)$};
	\draw (c1) -- node[inner sep=1pt, pin={[yshift=-2.4mm]90:{$\{Hot,AttC(A)\}$}}]		{} (c2);
\end{tikzpicture}
}
\caption{FO jtree for $G^{ex}$ (local models in grey)}
\label{fig:fojt}
\end{figure}

\ac{ljt} constructs an \ac{fojt} using a \ac{fodt}, enters evidence in the \ac{fojt}, and passes messages through an \emph{inbound} and an \emph{outbound} pass, to distribute local information of the nodes through the \ac{fojt}.
To compute a message, \ac{ljt} eliminates all non-seperator \acp{prv} from the parcluster's local model and received messages.
After message passing, \ac{ljt} answers queries.
For each query, LJT finds a parcluster containing the query term and sums out all non-query terms in its local model and received messages.

\Cref{fig:fojt} shows an \ac{fojt} of $G^{ex}$ with the local models of the parclusters and the separators as labels of edges.
During the \emph{inbound} phase of message passing, \ac{ljt} sends messages from $\mathbf{C}^1$ to $\mathbf{C}^2$ and for the \emph{outbound} phase a message from $\mathbf{C}^2$ to $\mathbf{C}^1$.
If we want to know whether $Hot$ holds, we query for $P(Hot)$ for which \ac{ljt} can use either parcluster $\mathbf{C}^1$ or $\mathbf{C}^2$.
Thus, \ac{ljt} can sum out $AttC(A)$ and $DoR(X)$ from $\mathbf{C}^2$'s local model $G^{2}$, $\{g^1\}$, combined with the received messages, here, one message from $\mathbf{C}^1$. 

\subsection{LDJT: Overview} 

\ac{ldjt} efficiently answers queries $P(\mathbf{Q}_t|\mathbf{E}_{0:t})$, with a set of query terms $\{\mathbf{Q}_t\}_{t=0}^T$, given a \ac{pdm} $G$ and evidence $\{\mathbf{E}_t\}_{t=0}^T$, by performing the following steps:
\begin{enumerate*}[label=(\roman*)]
    \item Construct offline two \acp{fojt} $J_0$ and $J_t$ with \emph{in-} and \emph{out-clusters} from $G$.
    \item For $t=0$, using $J_0$ to enter $\mathbf{E}_0$, pass messages, answer each query term $Q_\pi^i \in \mathbf{Q}_0$, and preserve the state.
    \item For $t>0$, instantiate $J_t$ for the current time step $t$, recover the previous state, enter $\mathbf{E}_t$ in $J_t$, pass messages, answer each query term $Q_\pi^i \in \mathbf{Q}_t$, and preserve the state.
\end{enumerate*}

Next, we show how \ac{ldjt} constructs the \acp{fojt} $J_0$ and $J_t$ with \emph{in-} and \emph{out-clusters}, which contain a minimal set of \acp{prv} to m-separate the \acp{fojt}.
M-separation means that information about these \acp{prv} make \acp{fojt} independent from each other.
Afterwards, we present how \ac{ldjt} connects the \acp{fojt} for reasoning to solve the filtering and prediction problems efficiently.

\subsection{LDJT: FO Jtree Construction for PDMs}\label{ldjt:const} 

\ac{ldjt} constructs \acp{fojt} for $G_0$ and $G_\rightarrow$, both with an incoming and outgoing interface.
To be able to construct the interfaces in the \acp{fojt}, \ac{ldjt} uses the \ac{pdm} $G$ to identify the interface \acp{prv} $\mathbf{I}_t$ for a time slice $t$.
\begin{definition}
    The forward interface is defined as $\mathbf{I}_{t} = \{A_{t}^i \mid \exists \phi(\mathcal{A}) | C \in G :  A_{t}^i \in \mathcal{A} \wedge \exists A_{t+1}^j \in \mathcal{A}\}$, i.e., the \acp{prv} which have successors in the next slice. 
\end{definition}

For $G_{\rightarrow}^{ex}$, which is shown in \cref{fig:TSPG}, \acp{prv} $Hot_{t-1}$ and $Pub_{t-1}(X,P)$ have successors in the next time slice, making up $\mathbf{I}_{t-1}$.
To ensure interface \acp{prv} $\mathbf{I}$ ending up in a single parcluster, \ac{ldjt} adds a \ac{pf} $g^I$ over the interface to the model.
Thus, \ac{ldjt} adds a \ac{pf} $g^I_0$ over $\mathbf{I}_0$ to $G_0$, builds an \ac{fojt} $J_0$ and labels the parcluster with $g^I_0$ from $J_0$ as \emph{in-} and \emph{out-cluster}.
For $G_\rightarrow$, \ac{ldjt} removes all non-interface \acp{prv} from time slice $t-1$, adds \acp{pf} $g^I_{t-1}$ and $g^I_{t}$, constructs $J_t$.
Further, \ac{ldjt} labels 
the parcluster containing $g^I_{t-1}$ as \emph{in-cluster} and labels the parcluster containing $g^I_{t}$ as \emph{out-cluster}.

The interface \acp{prv} are a minimal required set to m-separate the \acp{fojt}.
\ac{ldjt} uses these \acp{prv} as separator to connect the \emph{out-cluster} of $J_{t-1}$ with the \emph{in-cluster} of $J_t$, allowing to reusing the structure of $J_t$ for all $t>0$.

\subsection{LDJT: Proceeding in Time with the FO Jtree Structures}\label{sec:forward}

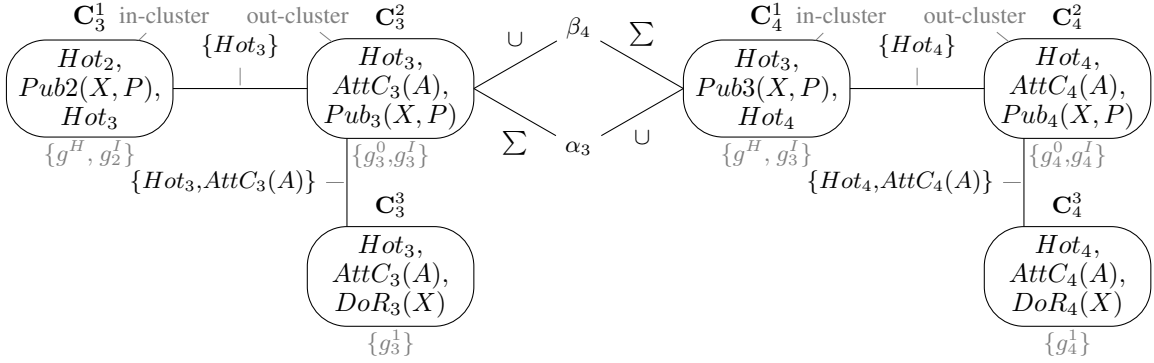
\begin{figure*}[t]
\center
\scalebox{1}{
\begin{tikzpicture}[every node/.style={font=\footnotesize}, node distance=40mm]
    
	\node[pc, label={[gray, inner sep=1pt]270:{$\{g^H,$ $g^I_2\}$}},    pin={[pin distance=1mm, gray, align=center]70:{in-cluster}}, label={[font=]90:{$\mathbf{C}^1_3$}}]				(c1) at (-10,0) {$Hot_{2},$\\ $Pub{2}(X,P),$ \\$Hot_{3}$};
    
	\node[pc, right of=c1,     pin={[pin distance=1mm, gray, align=center]120:{out-cluster}},
label={[gray, inner sep=1pt]270:{$\{g^0_3, $$g^I_3 \}$}},label={90:{$\mathbf{C}^2_3$}}]	(c2) {$Hot_{3},$\\$ AttC_{3}(A),$ \\ $Pub_{3}(X,P)$};
    \node[pc, below of=c2, node distance=2.5cm, label={[gray, inner sep=1pt]270:{$\{g^1_3\}$}},label={90:{$\mathbf{C}^3_3$}}]	(c3) {$Hot_{3},$\\$ AttC_{3}(A),$ \\ $DoR_{3}(X)$};
    
    \node[right of = c2, node distance=2.5cm] (c) {};
    \node[below of = c, node distance=0.8cm] (c11) {$\alpha_3$};
    \node[above of = c, node distance=0.8cm] (b) {$\beta_4$};
    
	\node[pc, right of=c, node distance=2.5cm, label={[gray, inner sep=1pt]270:{$\{g^H,$ $g^I_3\}$}},  pin={[pin distance=1mm, gray, align=center]70:{in-cluster}}, label={[font=]90:{$\mathbf{C}^1_4$}}]				(c4) {$Hot_{3},$\\ $Pub{3}(X,P),$ \\$Hot_{4}$};
	\node[pc, right of=c4,     pin={[pin distance=1mm, gray, align=center]120:{out-cluster}},
label={[gray, inner sep=1pt]270:{$\{g^0_4, $$g^I_4 \}$}},label={90:{$\mathbf{C}^2_4$}}]	(c5) {$Hot_{4},$\\$ AttC_{4}(A),$ \\ $Pub_{4}(X,P)$};
    \node[pc, below of=c5, node distance=2.5cm, label={[gray, inner sep=1pt]270:{$\{g^1_4\}$}},label={90:{$\mathbf{C}^3_4$}}]	(c6) {$Hot_{4},$\\$ AttC_{4}(A),$ \\ $DoR_{4}(X)$};

	\draw (c1) -- node[inner sep=1pt, pin={[yshift=-2.4mm]90:{$\{Hot_3\}$}}]		{} (c2);
	\draw (c2.after south west) -- node[inner sep=1pt, pin={[xshift=2.4mm]180:{$\{Hot_3, $$ AttC_3(A)\}$}}]		{} (c3.before north west);

	\draw (c4) -- node[inner sep=1pt, pin={[yshift=-2.4mm]90:{$\{Hot_4\}$}}]		{} (c5);
	\draw (c5.after south west) -- node[inner sep=1pt, pin={[xshift=2.4mm]180:{$\{Hot_4, $$ AttC_4(A)\}$}}]		{} (c6.before north west);

    \path [-] (c11) edge node [label= below:{$\sum$}] {} (c2.east);
    \path [-] (c11) edge node [label= below:{$\cup$}] {} (c4.west);
    \path [-] (b) edge node [label= above:{$\cup$}] {} (c2.east);
    \path [-] (b) edge node [label= above:{$\sum$}] {} (c4.west);

\end{tikzpicture}
}
\caption{Forward and backward pass of \ac{ldjt} (local models and labeling in grey)}
\label{fig:fojt1}	
\end{figure*}

Since $J_0$ and $J_t$ are static, \ac{ldjt} uses \ac{ljt} as a subroutine by passing on a constructed \ac{fojt}, queries, and evidence for step $t$ to handle evidence entering, message passing, and query answering using the \ac{fojt}.
Further, for proceeding to the next time step, \ac{ldjt} calculates an $\alpha_t$ message over the interface \acp{prv} using the \emph{out-cluster} to preserve the information about the current state.
Afterwards, \ac{ldjt} increases $t$ by one, instantiates $J_t$, and adds $\alpha_{t-1}$ to the \emph{in-cluster} of $J_t$.
During message passing, $\alpha_{t-1}$ is distributed through $J_t$.

\Cref{fig:fojt1} depicts how \ac{ldjt} uses the interface message passing between time step three to four.
First, \ac{ldjt} sums out the non-interface \ac{prv} $AttC_3(A)$ from $\mathbf{C}_3^2$'s local model and the received messages and saves the result in message $\alpha_3$.
After increasing $t$ by one, \ac{ldjt} adds $\alpha_3$ to the \emph{in-cluster} of $J_4$, $\mathbf{C}_4^1$.
$\alpha_3$ is then distributed by message passing and accounted for during calculating $\alpha_4$.

    \section{Preventing Groundings in LJT}

A lifted solution to a query given a model means that we compute an answer without grounding a part of the model. 
Unfortunately, not all models have a lifted solution because \ac{lve}, the basis for \ac{ljt}, requires certain conditions to hold.
Therefore, these models involve groundings with any exact lifted inference algorithm. 
Grounding a \ac{lv} is expensive and, during message passing, may propagate through all nodes.
\ac{ljt} has a few approaches to prevent groundings for a static \ac{fojt}.
On the one hand, some approaches originate from \ac{lve}.
On the other hand, \ac{ljt} has a fuse operator to prevent groundings, occurring due to a non-ideal elimination order.
Finding an optimal elimination order is in general NP-hard \cite{darwiche2009modeling}.
This section is mainly based on \cite{braun2017preventing}.

\subsection{General Grounding Prevention Techniques from LVE}

One approach to prevent groundings is to perform lifted summing out.
The idea is to compute VE for one case and exponentiate the result for isomorphic instances.
Another approach in \ac{lve} to prevent groundings is count-conversion, which exploits that all \acp{rv} of a \ac{prv} $A$ evaluate to a value $v$ of $range(A)$.
\ac{lve} forms a histogram by counting for each $v \in range(A)$ how many instances of $gr(A)$ evaluate to $v$.
Let us start by defining \ac{crv}.
\begin{definition}
    $\#_{X \in C}[P(\mathbf{X})]$ denotes a \ac{crv} with \ac{prv} $P(\mathbf{X})$ and constraint $C$, where $lv(\mathbf{X}) = \{X\}$.
    Its range is the space of possible histograms. 
    If $\{X\} \subset lv(\mathbf{X})$, the \ac{crv} is a parameterised CRV (PCRV) representing a set of \acp{crv}. 
    Since counting binds \ac{lv} $X$, $lv(\#_{X \in C}[P(\mathbf{X})]) = \mathbf{X} \setminus \{X\}$.
    We count-convert a \ac{lv} $X$ in a \ac{pf} $g = \mathbf{L} : \phi(\mathcal{A})|C$ by turning a \ac{prv} $A^i \in \mathcal{A}$, $X \in lv(A^i)$, into a \ac{crv} $A^{i'}$.
    In the new \ac{pf} $g^\prime$, the input for $A^{i'}$ is a histogram $h$. 
    Let $h(a^i)$ denote the count of $a^i$ in $h$. 
    Then, $\phi^\prime(...,a^{i-1},h,a^{i+1},...)$ maps to  $\prod_{a^i \in range(A^i)} \phi(...,a^{i-1},a^i,a^{i+1},...)^{h(a^i)}$.
\end{definition}
One precondition to count-convert a \ac{lv} $X$ in $g$, is that only one input in $g$ contains $X$.
To perform lifted summing out \ac{prv} $A$ from \ac{pf} $g$, $lv(A) = lv(g)$.
For the complete list of preconditions for both approaches, see \cite{TagFiDaBl13}.

\subsection{Preventing Groundings during Intra FO Jtree Message Passing}\label{sec:fusion}

During message passing, \ac{ljt} eliminates \acp{prv} by summing out. 
Thus, in case \ac{ljt} cannot apply lifted summing out, it has to ground \acp{lv}.
The messages \ac{ljt} passes via the separators restrict the elimination order, which can lead to grounding, in case lifted summing out is not applicable.

\ac{ljt} has three tests whether groundings occur during message passing.
Roughly speaking, the first test checks if \ac{ljt} can apply lifted summing out, the second test checks to prevent groundings by count-conversion, and the third test validates that a count-conversion will not result in groundings in another parcluster.

During message passing, a parcluster $\mathbf{C}^i = \mathcal{A}^i|C^i$ sends a message $m^{ij}$ containing the \acp{prv} of the separator $\mathbf{S}^{ij}$ to parcluster $\mathbf{C}^j$.
To calculate the message $m^{ij}$, \ac{ljt} eliminates the parcluster \acp{prv} not part of the separator, i.e., $\mathbf{E}^{ij} := \mathcal{A}^i \setminus \mathbf{S}^{ij}$, from the local model and all messages received from other nodes than $j$, i.e., $G^\prime := G^i \cap \{m^{il}\}_{l \neq j}$. 
To eliminate a \ac{prv} from $G^\prime$, \ac{ljt} has to eliminate the \ac{prv} from all \acp{pf} of $G^\prime$.
By combining all these \acp{pf}, \ac{ljt} only has to check whether a lifted summing out is possibile to eliminate the \ac{prv} for all \acp{pf}.
To eliminate $E \in \mathbf{E}^{ij}$ by lifted summing out from $G^\prime$, we replace all \acp{pf} $g \in G^\prime$ that include $E$ with a \ac{pf} $g^E = \phi(\mathcal{A}^E)|C^E$ that is the lifted product or the combination of these \acp{pf}. 
Let $\mathbf{S}^{ij^E} := \mathbf{S}^{ij} \cap \mathcal{A}^E$ be the set of \acp{rv} in the separator that occur in $g^E$. 
For lifted message calculation, it necessarily has to hold $\forall S \in \mathbf{S}^{ij^E}$,
\begin{equation}\label{eq:1}
    lv(S) \subseteq lv(E).
\end{equation}
Otherwise, $E$ does not include all \acp{lv} in $g^E$. 
\ac{ljt} may induce \cref{eq:1} for a particular $S$ by count conversion if $S$ has an additional, count-convertible \ac{lv}:
\begin{equation}\label{eq:2}
    lv(S) \setminus lv(E) = \{L\}, \text{ L count-convertible in }g^E.
\end{equation}
In case \cref{eq:2} holds, \ac{ljt} count-converts $L$, yielding a (P)CRV in $m^{ij}$, else, \ac{ljt} grounds.
Unfortunately, a (P)CRV can lead to groundings in another parcluster.
Hence, count-conversion helps in preventing a grounding if all following messages can handle the resulting (P)CRV. 
Formally, for each node $k$ receiving $S$ as a (P)CRV with counted \ac{lv} $L$, it has to hold for each neighbour $n$ of $k$ that

\begin{equation}\label{eq:3}
    S \in \mathbf{S}^{kn} \vee\text{ L count-convertible in }g^S.
\end{equation}
\ac{ljt} fuses two parclusters to prevent groundings if \cref{eq:1,eq:2,eq:3} checks determine groundings would occur by message passing between these two parcluster.

    \section{Preventing Groundings in LDJT}

Unnecessary groundings have a huge impact on temporal models, as groundings during message passing can propagate through the complete model, basically turing it into the ground model.
\ac{ldjt} has an intra and inter \ac{fojt} message passing phase.
Intra \ac{fojt} message passing takes place inside of an \ac{fojt} for one time step.
Inter \ac{fojt} message passing takes place between two \acp{fojt}.
To prevent groundings during intra \ac{fojt} message passing, \ac{ljt} successfully proposes to fuse parclusters \cite{braun2017preventing}.
Unfortunately, having two \acp{fojt}, \ac{ldjt} cannot fuse parclusters from different \acp{fojt}.
Hence, \ac{ldjt} requires a different approach to prevent unnecessary groundings during inter \ac{fojt} message passing.

In the following, we present how \ac{ldjt} prevents grounding and discuss the combination of preventing groundings during both intra and inter \ac{fojt} message passing as well as the implications for a lifted run.

\subsection{Preventing Groundings during Inter FO Jtree Message Passing}

As we desire a lifted solution, \ac{ldjt} also needs to prevent unnecessary groundings induced during inter \ac{fojt} message passes.
Therefore, \ac{ldjt}'s \emph{expanding} performs two steps:
\begin{enumerate*}
    \item check whether inter \ac{fojt} message pass induced groundings occur, 
    \item prevent groundings by extending the set of interface \acp{prv}, and
    prevent possible intra \ac{fojt} message pass induced groundings. 
\end{enumerate*}

\subsubsection{Checking for Groundings}

To determine whether an inter \ac{fojt} message pass induces groundings, \ac{ldjt} also uses
\cref{eq:1,eq:2,eq:3}.
For the forward pass, \ac{ldjt} applies the equations to check whether the $\alpha_{t-1}$ message from $J_{t-1}$ to $J_t$ leads to groundings.
More precisely, \ac{ldjt} needs to check for groundings for the inter \ac{fojt} message passing between $J_0$ and $J_1$ as well as between two temporal \ac{fojt} copy patters, namely $J_{t-1}$ to $J_t$ for $t > 1$.

Thus, \ac{ldjt} checks all \acp{prv} $E \in \mathbf{E}^{ij}$, where $i$ is the \emph{out-cluster} from $J_{t-1}$ and $j$ is the \emph{in-cluster} from $J_t$, for groundings.
In case \cref{eq:1} holds, no additional checks for $E$ are necessary as eliminating $E$ does not induce groundings.
In case \cref{eq:2} holds, \ac{ldjt} has to test whether \cref{eq:3} holds in $J_t$ at least on the path from \emph{in-cluster} to \emph{out-cluster}. 
Hence, if \cref{eq:2,eq:3} both hold, eliminating $E$ does not lead to groundings, but if  \cref{eq:2} or \cref{eq:3} fail groundings occur during message passing.  

\subsubsection{Expanding Interface Separators}

In case eliminating $E$ leads to groundings, \ac{ldjt} delays the elimination to a point where the elimination does no longer lead to groundings. 
Therefore, \ac{ldjt} adds $E$ to the \emph{in-cluster} of $J_t$, which results in $E$ also being added to the inter \ac{fojt} separator .
Hence, \ac{ldjt} does not need to eliminate $E$ in the \emph{out-cluster} of $J_{t-1}$ anymore.
Based on the way \ac{ldjt} constructs the \ac{fojt} structures, the \acp{fojt} stay valid.
Lastly, \ac{ldjt} prevents groundings in the extended \emph{in-cluster} of $J_t$ as described in \cref{sec:fusion}.

Let us now have a look at \cref{fig:fojt1} to understand the central idea of preventing inter \ac{fojt} message pass induced groundings.
\cref{fig:fojt1} shows $J_t$ instantiated for time step $3$ and $4$.
Using these instantiations, \ac{ldjt} checks for groundings during inter \ac{fojt} message passing for the temporal copy pattern.
To compute $\alpha_3$, \ac{ldjt} eliminates $AttC_3(A)$ from $\mathbf{C}^2_3$'s local model.
Hence, \ac{ldjt} checks whether the elimination leads to groundings.
In this example, \cref{eq:1} does not hold, since $AttC_3(A)$ does not contain all \acp{lv}, $X$ and $P$ are missing.
Additionally, \cref{eq:2} is not applicable, as the expression $lv(S) \setminus lv(E) = \{X, P\} \setminus \{C\} = \{X, P\}$, which contains more than one \ac{lv} and therefore is not count-convertible.

As eliminating $AttC_3(A)$ leads to groundings, \ac{ldjt} adds $AttC_3(A)$ to the parcluster $\mathbf{C}^1_4$.
Additionally, \ac{ldjt} also extends the inter \ac{fojt} separator with $AttC_3(A)$ and thereby changes the elimination order.
By doing so, \ac{ldjt} does not need to eliminate $AttC_3(A)$ in $\mathbf{C}^2_3$ anymore and therefore calculating $\alpha_3$ does not lead to groundings.
However, \ac{ldjt} has to check whether adding the \ac{prv} leads to groundings in $\mathbf{C}^1_4$.
For the extended parcluster $\mathbf{C}^1_4$, \ac{ldjt} needs to eliminate the \acp{prv} $Hot_3$, $AttC_3(A)$, and $Pub3(X,P)$.
To eliminate $Pub3(X,P)$, \ac{ldjt} first count-converts $AttC_3(A)$ and then \cref{eq:1} holds for $Pub3(X,P)$.
Afterwards, it can eliminate the count-converted $AttC_3(A)$ and the \ac{prv} $Hot_3$ as \cref{eq:1} holds for both of them.
Thus, by adding the \ac{prv} $AttC_{t-1}(A)$ to the \emph{in-cluster} of $J_t$ and thereby to the inter \ac{fojt} separator, \ac{ldjt} can prevent unnecessary groundings. 
Additionally, as \ac{ldjt} uses this \ac{fojt} structure for all time steps $t > 0$, i.e., the changes to the structure also hold for all $t > 0$.

\begin{theorem}
    \ac{ldjt}'s \emph{expanding} is correct and produces a valid \ac{fojt}.
\end{theorem}
  
\begin{proof}
    After \ac{ldjt} creates the \ac{fojt} structures initially, the separator between \ac{fojt} $J_{t-1}$ and $J_t$ consists of exactly the \acp{prv} from $\mathbf{I}_{t-1}$.
    Thus, by taking the intersection of the \acp{prv} contained in $J_{t-1}$ and $J_t$, we get the set of \acp{prv} from $\mathbf{I}_{t-1}$.
    While \ac{ldjt} calculates $\alpha_{t-1}$, it only needs to eliminate \acp{prv} $\mathbf{E}$ not contained in the separator and thereby $\mathbf{I}_{t-1}$.
    Therefore, all $E \in \mathbf{E}$ are not contained in any parcluster of $J_t$.
    Hence, by adding $E$ to the \emph{in-cluster} of $J_t$, \ac{ldjt} does not violate any \ac{fojt} properties.
    Further, \ac{ldjt} does not even have to validate properties like the running intersection property, since it could not have been violated in the first place. 
    Additionally, \ac{ldjt} extends the set of interface \acp{prv}, resulting in an over-approximation of the required \acp{prv} for the inter \ac{fojt} communication to be correct.
\end{proof}

\subsection{Discussion}

In the following, we start by discussing workload and performance aspects of the intra and inter \ac{fojt} message passing.
Afterwards, we present model constellations where \ac{ldjt} cannot prevent groundings and indicate the extension of the presented algorithm to a backward pass.
 
\subsubsection{Performance}
The additional workload for \ac{ldjt} introduced by handling unnecessary groundings is moderate.
In the best case, \ac{ldjt} checks \cref{eq:1,eq:2,eq:3} for calculating two messages, namely for the $\alpha_{t-1}$ message and for the message \ac{ldjt} passes from in \emph{in-cluster} of $J_t$ in the direction of the \emph{out-cluster} of $J_t$.
In the worst case, \ac{ldjt} needs to check $1 + (m-1)$ messages, where $m$ is the number of parclusters on the path from the \emph{in-cluster} to the \emph{out-cluster} in $J_t$.

From a performance point of view, increasing the size of the $\alpha$ messages and of a parcluster is not ideal, but always better than the impact of groundings.
By applying the intra \ac{fojt} message passing check, \ac{ldjt} may fuse the \emph{in-cluster} and \emph{out-cluster}, which most likely results in a parcluster with many model \acp{prv}. 
Increasing the number of \acp{prv} in a parcluster, increases \ac{ldjt}'s workload for query answering.
But even with the increased workload a lifted run is faster than grounding. 
However, in case the checks determine that a lifted solution is not obtainable, using the initial model with the local clustering is the best solution.

First, applying \ac{ljt}'s \emph{fusion} is more efficient since fusing the \emph{out-cluster} with another parclusters could increase the number of its \acp{prv}.
In case of changed \acp{prv}, \ac{ldjt} has to rerun the \emph{expanding} check.
Therefore, \ac{ldjt} first applies the intra and then the inter \ac{fojt} message passing checks.


\subsubsection{Groundings \ac{ldjt} Cannot Prevent}

Fusing the \emph{in-cluster} and \emph{out-cluster} due to the inter \ac{fojt} message passing check is one case for which \ac{ldjt} cannot prevent groundings.
In this case, \ac{ldjt} cannot eliminate $E$ in the \emph{out-cluster} of $J_{t-1}$ without groundings.
Thus, \ac{ldjt} adds $E$ to the \emph{in-cluster} of $J_t$.
The checks whether \ac{ldjt} can eliminate $E$ on the path from the \emph{in-cluster} to the \emph{out-cluster} of $J_t$ fail.
Thereby, \ac{ldjt} fuses all parclusters on the path between the two parclusters and \ac{ldjt} still cannot eliminate $E$.
Even worse, \ac{ldjt} cannot eliminate $E$ from time step $t-1$ and $t$ in the \emph{out-cluster} to calculate $\alpha_t$.
In theory, for an unrolled model, a lifted solution might be possible, but with many \acp{prv} in a parcluster, since, in addition to other \acp{prv}, one parcluster contains $E$ for all time steps.
Depending on the domain size and the maximum number of time steps, either grounding or using the unrolled model is advantageous.

If $S$ occurs in an inter-slice \ac{pf} for both time steps, then another source of groundings is a count-conversion of $S$ to eliminate $E$.
In such a case, \ac{ldjt} cannot count-convert $S$ in the inter-slice \ac{pf}, which leads to groundings.

\subsubsection{Extension}
So far, we focused on preventing groundings during a forward pass, which is the most crucial part as \ac{ldjt} needs to proceed forward in time.
\Cref{fig:fojt1} also indicates a backward pass during inter \ac{fojt} message passing.
Actually, the presented idea can be applied to a backward pass.
The proof also holds for the backward pass, since intersecting the sets of \acp{prv} of $J_{t-1}$ and $J_t$ only contains the \acp{prv} $\mathbf{I}_{t-1}$.
Therefore, if a \ac{prv} $E$ from $J_t$ is added to $J_{t-1}$, $E$ is not included in $J_{t-1}$ and thereby $J_{t-1}$ is still valid. 

    \begin{figure*}[t]
    \includegraphics[width=\textwidth]{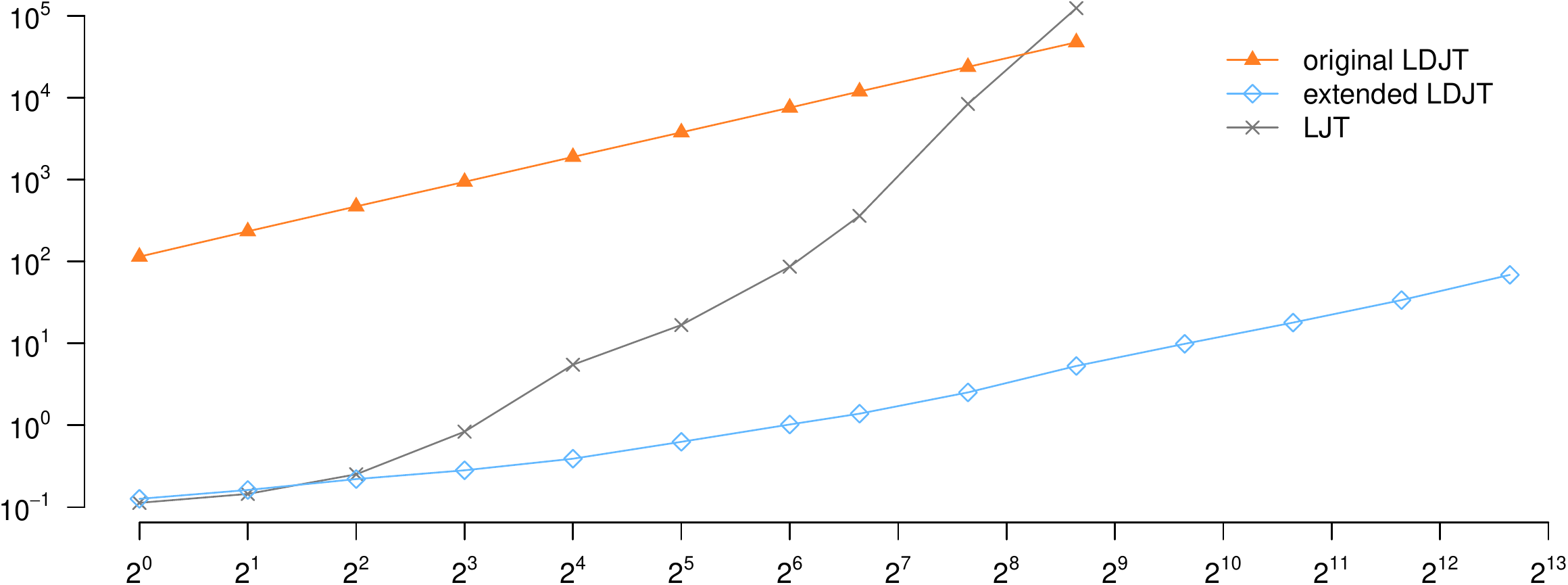}
  \caption{Y-axis: runtimes [seconds], x-axis: maximum time steps, both in log scale}
  \label{fig:times_overall}	
\end{figure*}

\section{Evaluation}\label{sec:eval}

For the evaluation, we use the example model $G^{ex}$ with the set of evidence being empty, for $|\mathcal{D}(X)| = 10 $, $|\mathcal{D}(P)| = 3 $, $|\mathcal{D}(C)| = 20$, and the queries
   $ \{Hot_t    ,      AttC_t(c_1)     ,            DoR_t(x_1)\}$
for each time step. 
We compare the runtimes on commodity hardware with 16 GB of RAM of the extended \ac{ldjt} version against \ac{ldjt}'s original version and then also against \ac{ljt} for multiple maximum time steps.

\Cref{fig:times_overall} shows the runtime in seconds for each maximum time step.
We can see that the runtime of the extended \ac{ldjt} (diamond) and the original \ac{ldjt} (filled triangle) is, as expected, linear, while the runtime of \ac{ljt} (cross) roughly is exponential, to answer queries for changing maximum number of time steps.
Further, we can see how crucial preventing groundings is.
Due to the \ac{fojt} construction overhead, the extended version is about a magnitude of three faster for first few time steps, but the construction overhead becomes negligible with more time steps.
Overall, the extended \ac{ldjt} is up to a magnitude of four faster.

Additionally, we see the runtimes of \ac{ljt}.
The runtimes with and without fusion are about the same and thus not distinguished.
\ac{ljt} is faster for the initial time steps, especially in case grounding are prevented by unrolling.
Nonetheless, after several time steps, the size of the parclusters becomes a big factor, which also explains the exponential behaviour \cite{taghipour2013first}.
To summarise the evaluation results, on the one hand, we see how crucial the prevention of groundings is and, on the other hand, how crucial the dedicated handling of temporal aspects is. 

    \section{Conclusion}\label{sec:conc}

We present how \ac{ldjt} can prevent unnecessary groundings by delaying eliminations to the next time step and thereby changing the elimination order.
To delay eliminations, \ac{ldjt} increases the \emph{in-cluster} of the temporal \ac{fojt} structure and the separator between \emph{out-cluster} and \emph{in-cluster} with \acp{prv}, which lead to the groundings.
Further, due to temporal m-separation, which is ensured by the \emph{in-} and \emph{out-clusters}, \ac{ldjt} reuses the same changed \ac{fojt} structure for all time steps $t > 0$.
First results show that the extended \ac{ldjt} significantly outperforms the orignal version and \ac{ljt} if unnecessary groundings occur.

We currently work on extending \ac{ldjt} to also calculate the most probable explanation. 
Other interesting future work includes a tailored automatic learning for \acp{pdm}, parallelisation of \ac{ljt}, and improved evidence entering.
    
\bibliographystyle{aaai}
\bibliography{../bib/tex/bib}

\begin{thebibliography}{}

\bibitem[\protect\citeauthoryear{Ahmadi \bgroup et al\mbox.\egroup
  }{2013}]{ahmadi2013exploiting}
Ahmadi, B.; Kersting, K.; Mladenov, M.; and Natarajan, S.
\newblock 2013.
\newblock {Exploiting Symmetries for Scaling Loopy Belief Propagation and
  Relational Training}.
\newblock {\em Machine learning} 92(1):91--132.

\bibitem[\protect\citeauthoryear{Braun and M{\"o}ller}{2016}]{BrMoe16a}
Braun, T., and M{\"o}ller, R.
\newblock 2016.
\newblock {Lifted Junction Tree Algorithm}.
\newblock In {\em Proceedings of the Joint German/Austrian Conference on
  Artificial Intelligence (K{\"u}nstliche Intelligenz)},  30--42.
\newblock Springer.

\bibitem[\protect\citeauthoryear{Braun and
  M{\"o}ller}{2017}]{braun2017preventing}
Braun, T., and M{\"o}ller, R.
\newblock 2017.
\newblock {Preventing Groundings and Handling Evidence in the Lifted Junction
  Tree Algorithm}.
\newblock In {\em Proceedings of the Joint German/Austrian Conference on
  Artificial Intelligence (K{\"u}nstliche Intelligenz)},  85--98.
\newblock Springer.

\bibitem[\protect\citeauthoryear{Braun and M\"oller}{2018}]{BraMo17}
Braun, T., and M\"oller, R.
\newblock 2018.
\newblock {Counting and Conjunctive Queries in the Lifted Junction Tree
  Algorithm}.
\newblock In {\em Postproceedings of the 5th International Workshop on Graph
  Structures for Knowledge Representation and Reasoning, GKR 2017, Melbourne,
  Australia, August 21, 2017}.
\newblock Springer.

\bibitem[\protect\citeauthoryear{Darwiche}{2009}]{darwiche2009modeling}
Darwiche, A.
\newblock 2009.
\newblock {\em Modeling and {R}easoning with {B}ayesian {N}etworks}.
\newblock Cambridge University Press.

\bibitem[\protect\citeauthoryear{de Salvo~Braz}{2007}]{Braz07}
de~Salvo~Braz, R.
\newblock 2007.
\newblock {\em Lifted {F}irst-{O}rder {P}robabilistic {I}nference}.
\newblock Ph.D. Dissertation, Ph. D. Dissertation, University of Illinois at
  Urbana Champaign.

\bibitem[\protect\citeauthoryear{Dign{\"o}s, B{\"o}hlen, and
  Gamper}{2012}]{dignos2012temporal}
Dign{\"o}s, A.; B{\"o}hlen, M.~H.; and Gamper, J.
\newblock 2012.
\newblock {Temporal Alignment}.
\newblock In {\em Proceedings of the 2012 ACM SIGMOD International Conference
  on Management of Data},  433--444.
\newblock ACM.

\bibitem[\protect\citeauthoryear{Dylla, Miliaraki, and
  Theobald}{2013}]{dylla2013temporal}
Dylla, M.; Miliaraki, I.; and Theobald, M.
\newblock 2013.
\newblock {A Temporal-Probabilistic Database Model for Information Extraction}.
\newblock {\em Proceedings of the VLDB Endowment} 6(14):1810--1821.

\bibitem[\protect\citeauthoryear{Gehrke, Braun, and
  M{\"o}ller}{2018}]{gehrke2018ldjt}
Gehrke, M.; Braun, T.; and M{\"o}ller, R.
\newblock 2018.
\newblock {Lifted Dynamic Junction Tree Algorithm}.
\newblock In {\em Proceedings of the 23rd International Conference on
  Conceptual Structures}.
\newblock Springer.
\newblock [to appear].

\bibitem[\protect\citeauthoryear{Geier and Biundo}{2011}]{geier2011approximate}
Geier, T., and Biundo, S.
\newblock 2011.
\newblock Approximate {O}nline {I}nference for {D}ynamic {M}arkov {L}ogic
  {N}etworks.
\newblock In {\em Proceedings of the 23rd IEEE International Conference on
  Tools with Artificial Intelligence (ICTAI)},  764--768.
\newblock IEEE.

\bibitem[\protect\citeauthoryear{Lauritzen and
  Spiegelhalter}{1988}]{lauritzen1988local}
Lauritzen, S.~L., and Spiegelhalter, D.~J.
\newblock 1988.
\newblock Local {C}omputations with {P}robabilities on {G}raphical {S}tructures
  and their {A}pplication to {E}xpert {S}ystems.
\newblock {\em Journal of the Royal Statistical Society. Series B
  (Methodological)}  157--224.

\bibitem[\protect\citeauthoryear{Manfredotti}{2009}]{manfredotti2009modeling}
Manfredotti, C.~E.
\newblock 2009.
\newblock {\em Modeling and {I}nference with {R}elational {D}ynamic {B}ayesian
  {N}etworks}.
\newblock Ph.D. Dissertation, Ph. D. Dissertation, University of
  Milano-Bicocca.

\bibitem[\protect\citeauthoryear{Milch \bgroup et al\mbox.\egroup
  }{2008}]{milch2008lifted}
Milch, B.; Zettlemoyer, L.~S.; Kersting, K.; Haimes, M.; and Kaelbling, L.~P.
\newblock 2008.
\newblock Lifted {P}robabilistic {I}nference with {C}ounting {F}ormulas.
\newblock In {\em Proceedings of AAAI}, volume~8,  1062--1068.

\bibitem[\protect\citeauthoryear{Murphy and Weiss}{2001}]{murphy2001factored}
Murphy, K., and Weiss, Y.
\newblock 2001.
\newblock The {F}actored {F}rontier {A}lgorithm for {A}pproximate {I}nference
  in {DBN}s.
\newblock In {\em Proceedings of the Seventeenth conference on Uncertainty in
  artificial intelligence},  378--385.
\newblock Morgan Kaufmann Publishers Inc.

\bibitem[\protect\citeauthoryear{Murphy}{2002}]{Murphy:2002:DBN}
Murphy, K.~P.
\newblock 2002.
\newblock {\em {Dynamic Bayesian Networks: Representation, Inference and
  Learning}}.
\newblock Ph.D. Dissertation, University of California, Berkeley.

\bibitem[\protect\citeauthoryear{Nitti, De~Laet, and
  De~Raedt}{2013}]{nitti2013particle}
Nitti, D.; De~Laet, T.; and De~Raedt, L.
\newblock 2013.
\newblock {A particle Filter for Hybrid Relational Domains}.
\newblock In {\em Proceedings of the IEEE/RSJ International Conference on
  Intelligent Robots and Systems (IROS)},  2764--2771.
\newblock IEEE.

\bibitem[\protect\citeauthoryear{Papai, Kautz, and
  Stefankovic}{2012}]{papai2012slice}
Papai, T.; Kautz, H.; and Stefankovic, D.
\newblock 2012.
\newblock Slice {N}ormalized {D}ynamic {M}arkov {L}ogic {N}etworks.
\newblock In {\em Proceedings of the Advances in Neural Information Processing
  Systems},  1907--1915.

\bibitem[\protect\citeauthoryear{Poole}{2003}]{poole2003first}
Poole, D.
\newblock 2003.
\newblock First-order probabilistic inference.
\newblock In {\em Proceedings of IJCAI}, volume~3,  985--991.

\bibitem[\protect\citeauthoryear{Richardson and
  Domingos}{2006}]{richardson2006markov}
Richardson, M., and Domingos, P.
\newblock 2006.
\newblock {Markov Logic Networks}.
\newblock {\em Machine learning} 62(1):107--136.

\bibitem[\protect\citeauthoryear{Taghipour \bgroup et al\mbox.\egroup
  }{2013}]{TagFiDaBl13}
Taghipour, N.; Fierens, D.; Davis, J.; and Blockeel, H.
\newblock 2013.
\newblock {{Lifted Variable Elimination: Decoupling the Operators from the
  Constraint Language}}.
\newblock {\em Journal of Artificial Intelligence Research} 47(1):393--439.

\bibitem[\protect\citeauthoryear{Taghipour, Davis, and
  Blockeel}{2013}]{taghipour2013first}
Taghipour, N.; Davis, J.; and Blockeel, H.
\newblock 2013.
\newblock First-order {D}ecomposition {T}rees.
\newblock In {\em Proceedings of the Advances in Neural Information Processing
  Systems},  1052--1060.

\bibitem[\protect\citeauthoryear{Thon, Landwehr, and
  De~Raedt}{2011}]{thon2011stochastic}
Thon, I.; Landwehr, N.; and De~Raedt, L.
\newblock 2011.
\newblock Stochastic relational processes: {E}fficient inference and
  applications.
\newblock {\em Machine Learning} 82(2):239--272.

\bibitem[\protect\citeauthoryear{Vlasselaer \bgroup et al\mbox.\egroup
  }{2014}]{vlasselaer2014efficient}
Vlasselaer, J.; Meert, W.; Van~den Broeck, G.; and De~Raedt, L.
\newblock 2014.
\newblock Efficient {P}robabilistic {I}nference for {D}ynamic {R}elational
  {M}odels.
\newblock In {\em Proceedings of the 13th AAAI Conference on Statistical
  Relational AI}, AAAIWS'14-13,  131--132.
\newblock AAAI Press.

\bibitem[\protect\citeauthoryear{Vlasselaer \bgroup et al\mbox.\egroup
  }{2016}]{vlasselaer2016tp}
Vlasselaer, J.; Van~den Broeck, G.; Kimmig, A.; Meert, W.; and De~Raedt, L.
\newblock 2016.
\newblock {TP-Compilation for Inference in Probabilistic Logic Programs}.
\newblock {\em International Journal of Approximate Reasoning} 78:15--32.

\end{thebibliography}

\end{document}